\documentclass{article} 
\usepackage{nips14submit_e,times}
\usepackage{hyperref}
\usepackage{url}
\usepackage{natbib}
\usepackage{amsmath}
\usepackage{amsthm}
\usepackage{amssymb}
\usepackage{graphicx}
\usepackage{xspace}
\usepackage{tabularx}
\usepackage{algorithm}
\usepackage{algorithmic}

\newtheorem{proposition}{Proposition}

\newif\ifarxiv
\arxivtrue

\newcommand{\bmx}[0]{\begin{bmatrix}}
\newcommand{\emx}[0]{\end{bmatrix}}

\nipsfinalcopy 

\title{Deep Directed Generative Autoencoders}

\author{Sherjil Ozair \\
Indian Institute of Technology Delhi \\
\And Yoshua Bengio\\
Universit\'e de Montr\'eal \\
CIFAR Fellow \\
}

\begin{document}

\maketitle

\begin{abstract}
For discrete data, the likelihood $P(x)$ can be rewritten exactly and
parametrized into $P(X=x)=P(X=x|H=f(x)) P(H=f(x))$ if $P(X|H)$ has enough
capacity to put no probability mass on any $x'$ for which $f(x')\neq f(x)$, 
where $f(\cdot)$ is a deterministic discrete function. The log of the
first factor gives rise to the log-likelihood reconstruction error of an
autoencoder with $f(\cdot)$ as the encoder and $P(X|H)$ as the
(probabilistic) decoder. The log of the second term can be seen as a
regularizer on the encoded activations $h=f(x)$, e.g., as in sparse
autoencoders. Both encoder and decoder can be represented by a deep neural
network and trained to maximize the average of the optimal log-likelihood $\log p(x)$. The objective
is to learn an encoder $f(\cdot)$ that maps $X$ to $f(X)$
that has a much simpler distribution than $X$ itself, estimated by $P(H)$. This ``flattens
the manifold'' or concentrates probability mass in a smaller number of
(relevant) dimensions over which the distribution factorizes. 
Generating samples from the model is straightforward using ancestral
sampling. One challenge
is that regular back-propagation cannot be used to obtain the
gradient on the parameters of the encoder, but we find that using
the straight-through estimator works well here.  We also find that although
optimizing a single level of such architecture may be difficult, much better
results can be obtained by pre-training and
stacking them, gradually transforming the data
distribution into one that is more easily captured by a simple parametric
model.
\end{abstract}

\section{Introduction}

Deep learning is an aspect of machine learning that regards the question of
learning multiple levels of representation, associated with different
levels of abstraction~\citep{Bengio-2009-book}.  These representations are
distributed~\citep{Hinton89b}, meaning that at each level there are many
variables or features, which together can take a very large number of
configurations. 

An important conceptual challenge of deep learning is the following
question: {\em what is a good representation}? The question is most
challenging in the unsupervised learning setup. Whereas we understand
that features of an input $x$ that are predictive of some target $y$
constitute a good representation in a supervised learning setting,
the question is less obvious for unsupervised learning. 

\subsection{Manifold Unfolding}

In this paper we explore
this question by following the geometrical inspiration introduced
by~\citet{Bengio-arxiv2014}, based on the notion of {\em manifold unfolding},
illustrated in Figure~\ref{fig:ddga}. It was already observed
by~\citet{Bengio-et-al-ICML2013} that representations obtained
by stacking denoising autoencoders or RBMs appear to yield ``flatter'' or
``unfolded'' manifolds:
if $x_1$ and $x_2$ are examples from the data generating distribution $Q(X)$ and $f$ is
the encoding function and $g$ the decoding function, 
then points on the line $h_\alpha = \alpha f(x_1) + (1-\alpha) f(x_2)$ ($\alpha \in [0,1]$)
were experimentally found to correspond to probable input configurations, i.e., $g(h_\alpha)$ looks like
training examples (and quantitatively often comes close to one). This property
is not at all observed for $f$ and $g$ being the identity function: interpolating
in input space typically gives rise to non-natural looking inputs (we can immediately
recognize such inputs as the simple addition of two plausible examples).
This is illustrated in Figure~\ref{fig:flattening-manifold}.
It means that the input manifold (near which the distribution concentrates)
is highly twisted and curved and occupies a small volume in input space.
Instead, when mapped in the representation space of stacked autoencoders (the output of $f$), 
we find that
the convex between high probability points (i.e., training examples) is often
also part of the high-probability manifold, i.e., the transformed manifold
is flatter, it has become closer to a convex set.

\begin{figure}[H]
\begin{center}
\centerline{\includegraphics[width=\columnwidth]{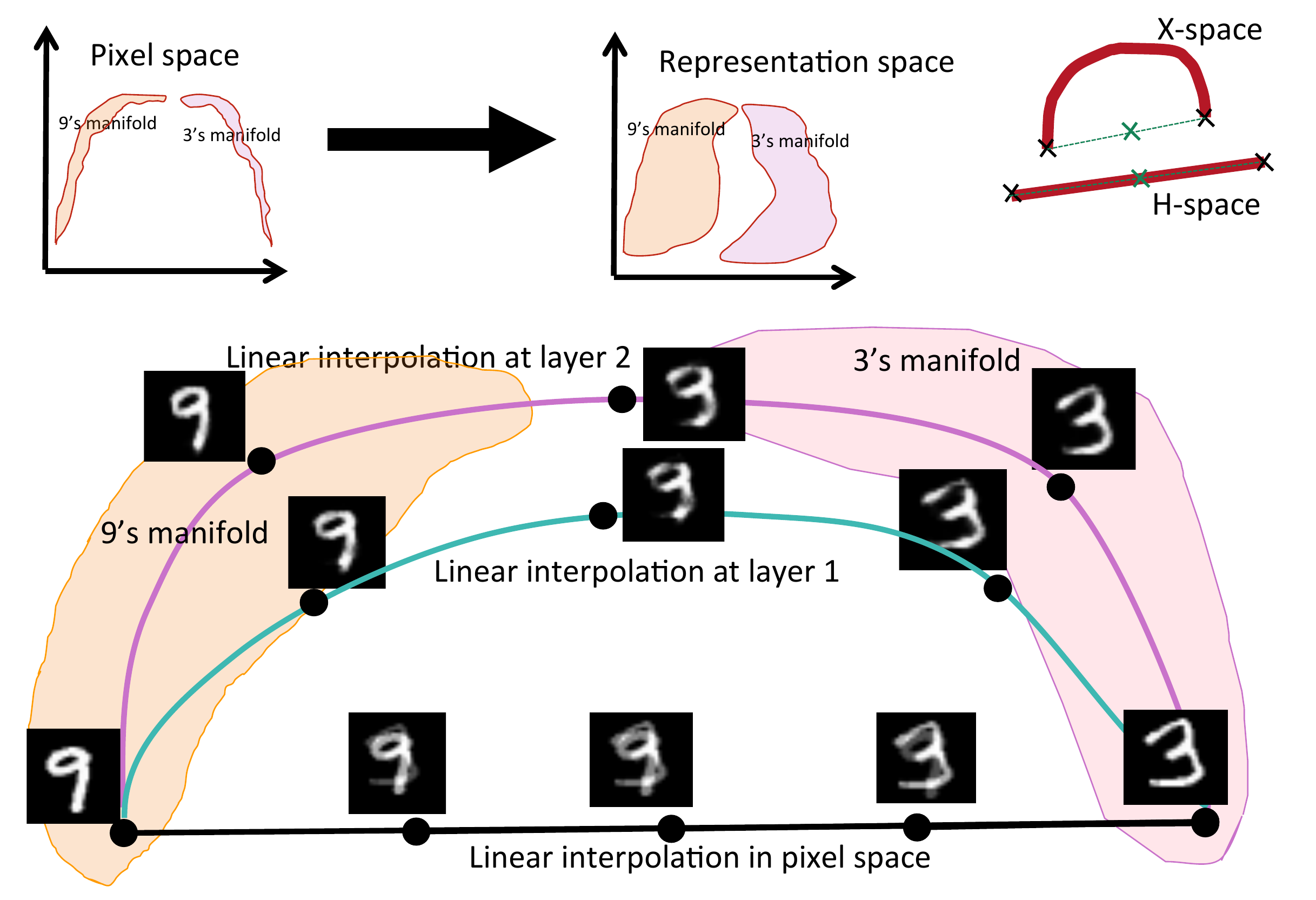}}
\caption{Illustration of the flattening effect observed in~\citet{Bengio-et-al-ICML2013}
by stacks of denoising autoencoders or RBMs, trained on MNIST digit images.
Whereas interpolating in pixel space ($X$-space)
between dataset examples (such as the 9 on the bottom
left and the 3 on the bottom right) gives rise to images unlike those in
the training set (on the bottom interpolation line), interpolating at the
first and second level of the stack of autoencoders ($H$-space) gives rise to images
(when projected back in input space, see text) that look like dataset examples.
These experiments suggest that the manifolds near which data concentrate, which are very
twisted and occupy a very small volume in pixel space, become flatter
and occupy more of the available volume in representation-space. Note
in particular how the two class manifolds have been brought closer to each other
(but interestingly there are also easier to separate, in representation space).
The manifolds associated with each class have become closer to a convex set, i.e., flatter.
}
\label{fig:flattening-manifold}
\end{center}
\end{figure} 

\subsection{From Manifold Unfolding to Probability Modeling}

If it is possible to unfold the data manifold into a nearly
flat and convex manifold (or set of manifolds), then estimating
the probability distribution of the data becomes much easier.
Consider the situation illustrated in Figure~\ref{fig:ddga}: a highly
curved 1-dimensional low-dimensional manifold is unfolded so that it occupies
exactly one dimension in the transformed representation (the ``signal dimension''
in the figure). Moving on the manifold corresponds to changing the hidden unit corresponding to
that signal dimension in representation space. On the other hand, moving orthogonal
to the manifold in input space corresponds to changing the ``noise dimension''.
There is a value of the noise dimension that corresponds to being on the
manifold, while the other values correspond to the volume filled by
unlikely input configurations. With the manifold learning mental picture,
estimating the probability distribution of the data basically amounts
to distinguishing between ``off-manifold'' configurations, which should have
low probability, from ``on-manifold'' configurations, which should have
high probability. Once the manifold is unfolded, answering that question
becomes very easy.

\begin{figure}[ht]
\begin{center}
\centerline{\includegraphics[width=0.35\columnwidth]{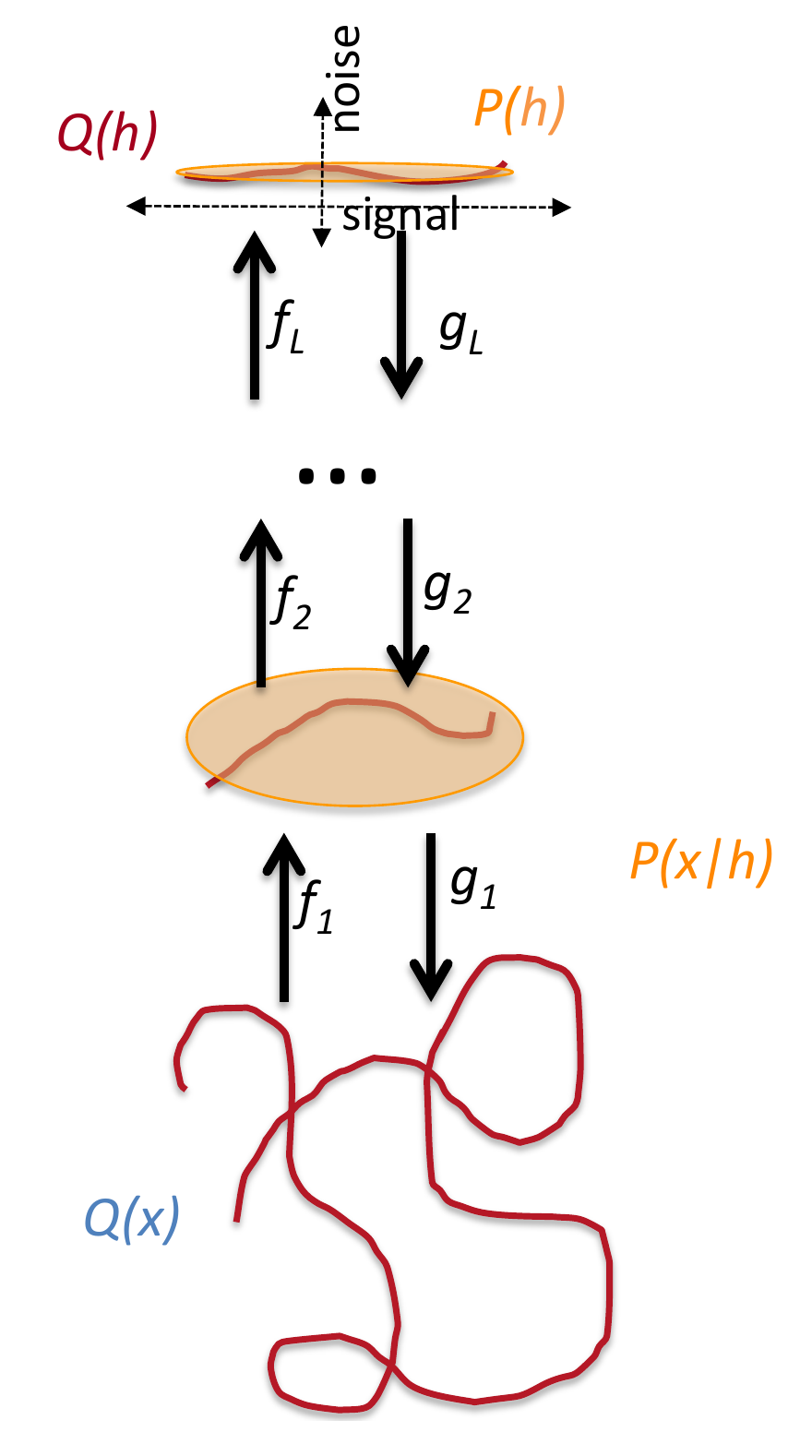}}
\caption{Illustration of the work that the composed encoder $f = f_L \circ \ldots f_2 \ldots f_1$
should do: flatten the manifold (more generally the region where probability mass is
concentrated) in such a way that a simple (e.g. factorized) prior distribution $P(h)$ can well approximate
the distribution of the transformed data $Q(h)$ obtained by applying $f$ to the data $x\sim Q(x)$.
The red curve indicates the manifold, or region near which $Q(x)$ concentrates. The first
encoder $f_1$ is not powerful and non-linear enough to untwist the manifold and flatten
it, so applying a factorized prior at that level would yield poor samples, because most
generated samples from the prior (in orange) would fall far from the transformed data
manifold (in red). At the top level, if training is successful, the transformed data distribution $Q(h)$
concentrates in a small number of directions (the ``signal'' directions, in the figure),
making it easy to distinguish the inside of the manifold (signal) from the outside (noise),
i.e. to concentrate probability mass in the right places.
At that point, $Q(h)$ and $P(h)$ can match better, and less probability mass is wasted
outside of the manifold.
}
\label{fig:ddga}
\end{center}
\end{figure} 

If probability mass is concentrated in a predictible and regular way
in representation space ($H$), it becomes easy to capture the data distribution $Q(X)$.
Let $H=f(X)$ be the random variable associated with the transformed data $X \sim Q(X)$,
with $Q(H)$ its marginal distribution.
Now, consider for example the extreme case where $Q(H)$ is almost factorized 
and can be well approximated by a factorized distribution model $P(H)$.
It means that the probability distribution in $H$-space can be captured
with few parameters and generalization is easy. Instead, when
modeling data in $X$-space, it is very difficult to find a family of
probability distributions that will put mass where the examples are
but not elsewhere. What typically happens is that we end up choosing
parameters of our probability model such that it puts a lot more
probability mass outside of the manifold than we would like. The
probability function is thus ``smeared''. We clearly get that effect
with non-parametric kernel density estimation, where we see that
the kernel bandwidth controls the amount of smearing, or smoothing.
Unfortunately, in a high-dimensional space, this puts a lot more
probability mass outside of the manifold than inside. Hence what
we see the real challenge as finding an encoder $f$ that transforms
the data distribution $Q(X)$ from a complex and twisted one into
a flat one, $Q(H)$, where the elements $H_i$ of $H$ are easy
to model, e.g. they are independent.

\section{Directed Generative Autoencoder (DGA) for Discrete Data}

We mainly consider in this paper the case of a discrete variable, which is simpler 
to handle.
In that case, a Directed Generative Autoencoder (DGA) is a model over the random variable $X$
whose training criterion is as follows
\begin{equation}
\label{eq:dga}
 \log P(X=x | H=f(x)) + \log P(H=f(X))
\end{equation}
where the deterministic function $f(\cdot)$ is called the encoder
and the conditional distribution $P(x|h)$ is called the decoder 
or decoding distribution. As shown below (Proposition~\ref{prop:exact-dga}), 
this decomposition becomes
exact as the capacity of the decoder $P(x|h)$ increases sufficiently to capture
a conditional distribution that puts zero probability on any $x'$ for which $f(x')\neq h$. 
The DGA is parametrized from two components:
\begin{enumerate}
\item $P(X=x | H=f(x))$ is an autoencoder with a {\em discrete}
representation $h=f(x)$. Its role in the log-likelihood
training objective is to make sure that $f$ preserves as much
information about $x$ as possible.
\item $P(H=h)$ is a probability model of the samples $H=f(X)$
obtained by sending $X$ through $f$. Its role in the log-likelihood
training objective is to make sure that $f$ transforms $X$
in a representation that has a distribution that can be well
modeled by the family of distributions $P(H)$.
\end{enumerate}

What can we say about this training criterion?

\begin{proposition}
\label{prop:exact-dga}
There exists a decomposition of the likelihood into \mbox{$P(X=x) = P(X=x | H=f(x)) P(H=f(X))$}
that is exact for discrete $X$ and $H$, and it is achieved when $P(x |h)$ is zero
unless $h=f(x)$. When $P(x|h)$ is trained with pairs $(X=x,H=f(x))$, it estimates
and converges (with enough capacity) to a conditional probability distribution which
satisfies this contraint. When $P(x|h)$ does not satisfy the condition, then
the unnormalized estimator \mbox{$P^*(X=x) = P(X=x|H=f(x))P(H=f(X))$} underestimates the true probability
\mbox{$P(X=x) = \sum_h P(x|H=h) P(H=h)$}.
\end{proposition}
\begin{proof}
We start by observing that because $f(\cdot)$ is deterministic, its value
$f(x)$ is perfectly predictible from the knowledge of $x$, i.e.,
\begin{equation}
  P(H=f(x) | X=x) = 1.
\end{equation}
Therefore, we can multiply this value of 1 by $P(X)$ and obtain the joint $P(X,H)$:
\begin{equation}
 P(X=x) = P(H=f(x) | X=x) P(X=x) = P(X=x, H=f(x))
\end{equation}
for any value of $x$.
Now it means that there exists a parametrization of the joint $P(X,H)$
into $P(H)P(X|H)$ that achieves
\begin{equation}
 P(X=x) = P(X=x | H=f(X)) P(H=f(X))
\end{equation}
which is the first part of the claimed result. Furthermore, for exact relationship
with the likelihood is achieved
when $P(x|h)=0$ unless $h=f(x)$, since with that condition,
\[
  P(X=x) = \sum_h P(X=x|H=h) P(H=h) = P(X=x|H=f(x)) P(H=f(x))
\]
because all the other terms of the sum vanish. Now, consider the case (which
is true for the DGA criterion) where the parameters
of $P(x|h)$ are only estimated to maximize the expected value of the
conditional likelihood $\log P(X=x | H=f(x))$. Because the maximum
likelihood estimator is consistent, if the family of distributions used
to estimate $P(x|h)$ has enough capacity to contain a solution
for which the condition of $P(x|h)=0$ for $h \neq f(x)$ is 
satisfied, then we can see that with enough capacity (which may also
mean enough training time), $P(x|h)$ converges to a solution that satisfies
the condition, i.e., for which $P(x) = P(x|H=f(x))P(H=f(x))$.
In general, a learned decoder $P(x|h)$ will not achieve
the guarantee that $P(x|h)=0$ when $h \neq f(x)$. However, the correct 
$P(x)$, for given $P(x|h)$ and $P(h)$ can always be written
\[
  P(x) = \sum_h P(x|h) P(h) \geq P(x|h=f(x)) P(H=f(x)) = P^*(x)
\]
which proves the claim that $P^*(x)$ underestimates the true likelihood.
\end{proof}

What is particularly interesting about this bound is that {\em as training
progresses and capacity increases} the bound becomes tight.
However, if $P(x|h)$ is a parametric distribution (e.g. factorized Binomial,
in our experiments) whose parameters (e.g. the Binomial probabilities)
are the output of a neural net, increasing the capacity of the neural net may not be sufficient
to obtain a $P(x|h)$ that satisfies the desired condition and makes the bound tight.
However, if $f(x)$ does not lose information about $x$, then $P(x|f(x))$ should
be unimodal and in fact be just 1. In other words, we should penalize the
reconstruction error term strongly enough to make the reconstruction error
nearly zero. That will guarantee that $f(x)$ keeps all the information about $x$
(at least for $x$ in the training set) and that the optimal $P(x|h)$ fits our
parametric family.

\subsection{Parameters and Training}

The training objective is thus a lower bound on the true log-likelihood:
\begin{equation}
\label{eq:log-likelihood}
 \log P(x) \geq \log P^*(x) = \log P(x|H=f(x)) + \log P(H=f(x))
\end{equation}
which can be seen to have three kinds of parameters:
\begin{enumerate}
\item The encoder parameters associated with $f$.
\item The decoder parameters associated with $P(x|h)$.
\item The prior parameters associated with $P(h)$.
\end{enumerate}

Let us consider how these three sets of parameters could be optimized
with respect to the log-likelihood bound (Eq.~\ref{eq:log-likelihood}).
The parameters of $P(h)$ can be learned by maximum likelihood (or any proxy for it),
with training examples that are the $h=f(x)$ when $x \sim Q(X)$
is from the given dataset. For example, if $P(H)$ is a factorized
Binomial, then the parameters are the probabilities \mbox{$p_i=P(H_i=1)$}
which can be learned by simple frequency counting of these events.
The parameters for $P(x|h)$ can be learned by maximizing the conditional
likelihood, like any neural network with a probabilistic output.
For example, if $X|H$ is a factorized Binomial, then we just have
a regular neural network taking $h=f(x)$ as input and $x$ as target,
with the cross-entropy loss function.

\subsection{Gradient Estimator for $f$}

One challenge is to deal with the training of the parameters of the
encoder $f$, because $f$ is discrete. We need to estimate a gradient
on these parameters in order to both optimize $f$ with respect to $\log P(h=f(x))$
(we want $f$ to produce outputs that can easily be modeled by the 
family of distributions $P(H)$) and with respect to $\log P(x|h=f(x))$
(we want $f$ to keep all the information about $x$, so as to be able
to reconstruct $x$ from $h=f(x)$). Although the true gradient
is zero, we need to obtain an update direction (a pseudo-gradient)
to optimize the parameters of the encoder with respect to these
two costs. A similar question was raised in a different context
in~\citet{Bengio-arxiv2013,bengio2013estimating}, and a number
of possible update directions were compared. In our experiments
we considered the special case where
\[
 f_i(x) = 1_{a_i(x)>0}
\]
where $a_i$ is the activation of the $i$-th output unit of the encoder before
the discretizing non-linearity is applied.
The actual encoder output is discretized, but we are interested in obtaining
a ``pseudo-gradient'' for $a_i$, which we will back-propagate inside
the encoder to update the encoder parameters.
What we did in the experiments is to compute the derivative of the reconstruction loss 
and prior loss \begin{equation}
\label{eq:loss}
{\cal L}=-\log P(x|h=f(x)) - \log P(h=f(x)),
\end{equation}
with respect to $f(x)$,
as if $f(x)$ had been continuous-valued.
We then used the {\bf Straight-Through Pseudo-Gradient}:
the update direction (or pseudo-gradient)
for $a$ is just set to be equal to the gradient with respect to $f(x)$:
\[
  \Delta a = \frac{\partial {\cal L}}{\partial f(x)}.
\]
The idea for this technique was proposed by~\citet{Hinton-Coursera2012} and
was used very succesfully in~\citet{bengio2013estimating}. It clearly has
the right sign (per value of $a_i(x)$, but not necessarily overall) but does not
take into account the magnitude of $a_i(x)$ explicitly.

Let us see how the prior negative log-likelihood bound can be written as a function of $f(x)$
in which we can pretend that $f(x)$ is continuous. For example, if
$P(H)$ is a factorized Binomial, we write this negative log-likelihood as the
usual cross-entropy:
\[
 - \sum_i f_i(x) \log P(h_i=1) + (1-f_i(x)) \log (1-P(h_i=1))
\]
so that if $P(x|h)$ is also a factorized Binomial, the overall loss is
\begin{align}
\label{eq:binomial-loss}
{\cal L} = & - \sum_i f_i(x) \log P(h_i=1) + (1-f_i(x)) \log (1-P(h_i=1)) \nonumber \\
           & - \sum_j x_j \log P(x_j=1|h=f(x)) + (1-x_j) \log (1-P(x_j=1|h=f(x)))
\end{align}
and we can compute $\frac{\partial {\cal L}}{\partial f_i(x)}$ as if $f_i(x)$ had been
a continuous-valued variable. All this is summarized in Algorithm~\ref{alg:dga}.

\begin{algorithm}[ht]
\caption{Training procedure for Directed Generative Autoencoder (DGA).}
\label{alg:dga}
\begin{algorithmic}
\STATE $\bullet$ Sample $x$ from training set.
\STATE $\bullet$ Encode it via feedforward network $a(x)$ and $h_i=f_i(x)=1_{a_i(x)>0}$.
\STATE $\bullet$ Update $P(h)$ with respect to training example $h$.
\STATE $\bullet$ Decode via decoder network estimating $P(x|h)$.
\STATE $\bullet$ Compute (and average) loss ${\cal L}$, as per Eq.~\ref{eq:loss},
e.g., in the case of factorized Binomials as per Eq.~\ref{eq:binomial-loss}.
\STATE $\bullet$ Update decoder in direction of gradient of $-\log P(x|h)$ w.r.t. the decoder parameters.
\STATE $\bullet$ Compute gradient $\frac{\partial \cal L}{\partial f(x)}$ as if $f(x)$ had been continuous.
\STATE $\bullet$ Compute pseudo-gradient w.r.t. $a$ as $\Delta a =  \frac{\partial \cal L}{\partial f(x)}$.
\STATE $\bullet$ Back-propagate the above pseudo-gradients (as if they were true gradients of the loss on $a(x)$)
inside encoder and update encoder parameters accordingly.
\end{algorithmic}
\end{algorithm}

\section{Greedy Annealed Pre-Training for a Deep DGA}

For the encoder to twist the data into a form that fits the prior $P(H)$,
we expect that a very strongly non-linear transformation will be required. As deep
autoencoders are notoriously difficult to train~\citep{martens2010hessian}, adding the extra constraint
of making the output of the encoder fit $P(H)$, e.g., factorial, was found
experimentally (and without surprise) to be difficult.

What we propose here is to use a an annealing (continuation method) and a greedy pre-training strategy,
similar to that previously proposed to train Deep Belief Networks 
(from a stack of RBMs)~\citep{Hinton06-small} or deep autoencoders (from a stack of shallow
autoencoders)~\citep{Bengio-nips-2006-small,Hinton+Salakhutdinov-2006}.

\subsection{Annealed Training}

Since the loss function of Eq.\ref{eq:dga} is hard to optimize directly, we consider a generalization
of the loss function by adding trade-off parameters to the two terms in the loss function corresponding
to the reconstruction and prior cost. A zero weight for the prior cost makes the loss function
same as that of a standard autoencoder, which is a considerably easier optimization problem.

\begin{align}
\label{eq:relaxed-binomial-loss}
{\cal L} = & - \beta \sum_i f_i(x) \log P(h_i=1) + (1-f_i(x)) \log (1-P(h_i=1)) \nonumber \\
           & - \sum_j x_j \log P(x_j=1|h=f(x)) + (1-x_j) \log (1-P(x_j=1|h=f(x))).
\end{align}

\subsubsection{Gradient Descent on Annealed Loss Function}

Training DGAs with fixed trade-off parameters is sometimes difficult because it is
much easier for gradient descent to perfectly optimize the prior cost by making $f$ map
all $x$ to a constant $h$. This may be a local optimum or a saddle point, escaping from which is 
difficult by gradient descent.
Thus, we use the tradeoff paramters to make the model first learn perfect reconstruction,
by setting zero weight for the prior cost. $\beta$ is then gradually increased to 1.
The gradual increasing schedule for $\beta$ is also important, as any rapid growth in $\beta$'s value
causes the system to `forget' the reconstruction and prioritize only the prior cost.

A slow schedule thus ensures that the model learns to reconstruct as well as to fit the prior.

\subsubsection{Annealing in Deep DGA}

Above, we describe the usefulness of annealed training of a shallow DGA. A similar trick
is also useful when pretraining a deep DGA. We use different values for these tradeoff parameters
to control the degree of difficulty for each
 pretraining stage. Initial stages have a high weight for reconstruction, and a low weight for 
 prior fitting, while the final stage has $\beta$ set to unity, giving back the 
 original loss function. Note that the lower-level DGAs can sacrifice on prior
 fitting, but must make sure that the reconstruction is near-perfect, so that no information is lost.
 Otherwise, the upper-level DGAs can never recover from that loss, which will show up in
 the high entropy of $P(x|h)$ (both for the low-level decoder and for the global decoder).

\section{Relation to the Variational Autoencoder (VAE) and Reweighted Wake-Sleep (RWS)}

The DGA can be seen as a special case of the Variational Autoencoder
(VAE), with various versions introduced
by~\citet{Kingma+Welling-ICLR2014,Gregor-et-al-ICML2014,Mnih+Gregor-ICML2014,Rezende-et-al-arxiv2014},
and of the Reweighted Wake-Sleep (RWS)
algorithm~\citep{Bornschein+Bengio-arxiv2014-small}.

The main difference between the DGA and these models is that with the latter the
encoder is stochastic, i.e., outputs a sample from an encoding (or approximate inference)
distribution $Q(h|x)$ instead of $h=f(x)$. This basically gives rise to a training
criterion that is not the log-likelihood but a variational lower bound on it,
\begin{equation}
\label{eq:vae-criterion}
  \log p(x) \geq E_{Q(h|x)}[ \log P(h) + \log P(x|h) - \log Q(h|x)].
\end{equation}
Besides the fact that $h$ is now sampled, we observe that the training criterion
has exactly the same first two terms as the DGA log-likelihood, but it also has
an extra term that attempts to maximize the conditional entropy of the encoder output,
i.e., encouraging the encoder to introduce noise, to the extent that it does not
hurt the two other terms too much. It will hurt them, but it will also help
the marginal distribution $Q(H)$ (averaged over the data distribution $Q(X)$)
to be closer to the prior $P(H)$, thus encouraging the decoder to contract
the ``noisy'' samples that could equally arise from the injected noise in $Q(h|x)$
or from the broad (generally factorized) $P(H)$ distribution.

\section{Experiments}
In this section, we provide empirical evidence for the feasibility of the proposed model, 
and analyze the influence of various techniques on the performance of the model.

We used the binarized MNIST handwritten digits dataset. We used the same binarized
version of MNIST as \cite{Uria+al-ICML2014}, and also used the same training-validation-test
split.

We trained shallow DGAs with 1, 2 and 3 hidden layers, and deep DGAs composed of
2 and 3 shallow DGAs. The dimension of $H$ was chosen to be 500, as this is 
considered sufficient for coding binarized MNIST. $P(H)$ is modeled as 
a factorized Binomial distribution.

Parameters of the model were learnt using minibatch gradient descent, with minibatch 
size of 100. Learning rates were chosen from {10.0, 1.0, 0.1}, and halved whenever
the average cost over an epoch increased. We did not use momentum or L1, L2 regularizers.
We used tanh activation functions in the hidden layers and sigmoid outputs.

While training, a small salt-and-pepper noise is added to the decoder input to make it
robust to inevitable mismatch between the encoder output $f(X)$ and samples from the prior, $h \sim P(H)$. Each bit of decoder input is selected with 1\% probability and changed to $0$ or $1$ randomly.

\begin{figure}[t]
\begin{center}
\begin{tabular}{c c}

\includegraphics[width=6.5cm]{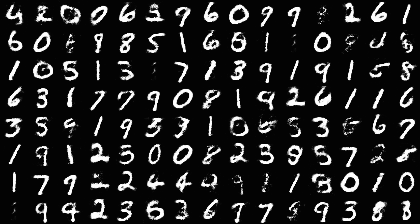} & \includegraphics[width=6.5cm]{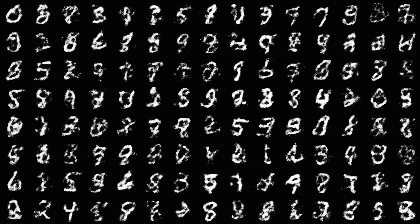} \\
(a) & (b)
\end{tabular}

\caption{(a) Samples generated from a 5-layer deep DGA, composed of 3 shallow DGAs 
of 3 layers (1 hidden layer) each. The model was trained greedily by training the 
first shallow DGA on the raw data and training each subsequent shallow DGA
on the output code of the previous DGA. (b) Sample generated from a 3-layer shallow DGA.}
\label{fig:samples}
\end{center}
\end{figure}

\subsection{Loglikelihood Estimator}
When the autoencoder is not perfect, i.e. the encoder and decoder are not
perfect inverses of each other, the loglikelihood estimates using
Eq. \ref{eq:binomial-loss} are biased estimates of the true loglikehood. We
treat these estimates as an unnormalized probability distribution, where
the partition function would be one when the autoencoder is perfectly
trained. In practice, we found that the partition function is less than
one. Thus, to compute the loglikehood for comparison, we estimate the
partition function of the model, which allows us to compute normalized
loglikelihood estimates. If $P^*(X)$ is the unnormalized probability
distribution, and $\pi(X)$ is another tractable distribution on $X$ from
which we can sample, the
partition function can be estimated by importance sampling:
$$Z = \sum_x P^*(x) = \sum_x \pi(x) \frac{P^*(x)}{\pi(x)} = E_{x \sim
  \pi(x)}[\frac{P^*(x)}{\pi(x)}]$$

As proposal distribution $\pi(x)$, We took $N$ expected values $\mu_j = E[X | H_j]$ under the decoder distribution,
for $H_j \sim P(H)$, and use them as centroids for a mixture model,

   $$\pi(x) = \frac{1}{N} \sum_{j=1}^N FactorizedBinomial(x; \mu_j).$$

Therefore,
$$log(P(X)) = log(P^*(X)) - log(Z)$$ gives us the estimated normalized loglikelihood.

\subsection{Performance of Shallow vs Deep DGA }
The 1-hidden-layer shallow DGA gave a loglikelihood estimate of -118.12 on the test set. The 5-layer deep DGA, 
composed of two 3-layer shallow DGAs trained greedily, gave a test set loglikelihood estimate of -114.29. We observed
that the shallow DGA had better reconstructions than the deep DGA. The deep DGA sacrificed on the reconstructibility, but 
was more successful in fitting to the factorized Binomial prior.

Qualitatively, we can observe in Figure~\ref{fig:samples} that samples from the deep DGA are much better than those from the shallow DGA.
The samples from the shallow DGA can be described as a mixture of incoherent MNIST features: although all the decoder units
have the necessary MNIST features, the output of the encoder does not match well the factorized Binomial of $P(H)$,
and so the decoder is not correctly mapping these unusual inputs $H$ from the Binomial prior to data-like samples.

The samples from the deep DGA are of much better quality. However, we can see that some of the samples are non-digits. This is due to the fact that the autoencoder had to sacrifice some reconstructibility to fit $H$ to the prior. The reconstructions
also have a small fraction of samples which are non-digits.


\subsection{Entropy, Sparsity and Factorizability}

We also compared the entropy of the encoder output and the raw data, under a factorized
Binomial model. The entropy, reported in Table~\ref{tab:entropy},
is measured with the logarithm base 2, so it counts
the number of bits necessary to encode the data under the simple factorized Binomial distribution.
A lower entropy under the factorized distribution means that fewer independent units are necessary to 
encode each sample. It means that the probability mass has been moved from a highly complex
manifold which is hard to capture under a factorized model and thus requires many dimensions
to characterize to a manifold that is aligned with a smaller set of dimensions, 
as in the cartoon of Figure~\ref{fig:ddga}. Practically this happens when many
of the hidden units take on a nearly constant value, i.e., the representation
becomes extremely ``sparse'' (there is no explicit preference for 0 or 1 for $h_i$ but one could
easily flip the sign of the weights of some $h_i$ in the output layer to make sure that 0 is
the frequent value and 1 the rare one). Table~\ref{tab:entropy} also contains a measure of sparsity of the 
representations based on such a bit flip (so as to make 0 the most frequent value).
This flipping allows to count the average number of 1's (of rare bits) necessary to
represent each example, in average (third column of the table).

We can see from Table~\ref{tab:entropy}
that only one autoencoder is not sufficient to reduce the entropy down
to the lowest possible value. Addition of the second autoencoder reduces
the entropy by a significant amount. 

Since the prior distribution is factorized, the encoder has to map data
samples with highly correlated dimensions, to a code with independent
dimensions. To measure this, we computed the Frobenius norm of the
off-diagonal entries of the correlation matrix of the data represented in its
raw form (Data) or at the outputs of the different encoders. See the 3rd columns
of Table~\ref{tab:entropy}. We see
that each autoencoder removes correlations from the data representation,
making it easier for a factorized distribution to model.

\begin{table}[H]
\begin{center}
\begin{tabular}{|c| c| c|c|}
\hline
Samples & Entropy & Avg \# active bits & $||Corr-diag(Corr)||_{F}$ \\
\hline
Data ($X$) & 297.6 & 102.1 & 63.5 \\
Output of $1^{st}$ encoder ($f_1(X)$) & 56.9  & 20.1 & 11.2\\
Output of $2^{nd}$ encoder ($f_2(f_1(X))$) & 47.6 & 17.4 & 9.4\\ 
\hline
\end{tabular}
\end{center}
\caption{Entropy, average number of active
bits (the number of rarely active bits, or 1's if 0 is the most frequent bit value, i.e. a measure of non-sparsity) 
and inter-dimension correlations, 
decrease as we encode into higher levels of representation,
making it easier to model it.}
\label{tab:entropy}
\end{table}

\section{Conclusion}

We have introduced a novel probabilistic interpretation for autoencoders as generative models,
for which the training criterion is similar to that of regularized (e.g. sparse) autoencoders
and the sampling procedure is very simple (ancestral sampling, no MCMC). We showed that this
training criterin is a lower bound on the likelihood and that the bound becomes tight as
the decoder capacity (as an estimator of the conditional probability $P(x|h)$) increases.
Furthermore, if the encoder keeps all the information about the input $x$, then the optimal
$P(x|h=f(x))$ is unimodal, i.e., a simple neural network with a factorial output suffices.
Our experiments showed that minimizing the proposed criterion yielded good generated
samples, and that even better samples could be obtained by pre-training a stack of
such autoencoders, so long as the lower ones are constrained to have very low reconstruction
error. We also found that a continuation method in which the weight of the prior term
is only gradually increased yielded better results. 

These experiments are most interesting because they reveal a picture of the representation
learning process that is in line with the idea of manifold unfolding and transformation from
a complex twisted region of high probability into one that is more regular (factorized)
and occupies a much smaller volume (small number of active dimensions). They also help
to understand the respective roles of reconstruction error and representation prior
in the training criterion and training process of such regularized auto-encoders.

\subsection*{Acknowledgments}

The authors would like to thank Laurent Dinh, Guillaume Alain and Kush Bhatia for fruitful discussions, and also the developers of Theano~\citep{bergstra+al:2010-scipy,Bastien-Theano-2012}. We acknowledge the support of the following agencies for
research funding and computing support: NSERC, Calcul Qu\'{e}bec, Compute Canada,
the Canada Research Chairs and CIFAR.

\bibliography{strings,strings-shorter,ml,aigaion}
\bibliographystyle{natbib}

\end{document}